\documentclass[letterpaper, 10 pt, conference]{ieeeconf}
\IEEEoverridecommandlockouts
\usepackage{amsmath,amsthm,amssymb,amsfonts,mathtools}
\usepackage{graphicx}
\usepackage{textcomp}
\usepackage{caption}
\usepackage{subcaption} 
\usepackage{hyperref}
\hypersetup{
    colorlinks=true,
    linkcolor=blue,
    filecolor=magenta,      
    urlcolor=cyan,
}
\usepackage{xcolor}
\DeclareMathOperator*{\argmin}{argmin}
\DeclareMathOperator*{\argmax}{argmax}

\DeclareMathOperator*{\E}{\mathbb{E}}

\newcommand{\abs}[1]{\left|#1\right|}
\newtheorem{prop}{Proposition}

\newcommand{\set}[1]{\left\{#1\right\}}
\newcommand{\brk}[1]{\left(#1\right)}
\newcommand{\bsq}[1]{\left[#1\right]}

\newcommand{\prob}[1]{\mathbb P\brk{#1}}
\newcommand{\R}{\mathbb{R}}

\DeclareMathAlphabet{\mymathbb}{U}{BOONDOX-ds}{m}{n}

\begin{document}

\title{Autonomous Driving With Perception Uncertainties: Deep-Ensemble Based Adaptive Cruise Control\\
\thanks{$^{*}$Xiao Li, Anouck Girard, and Ilya Kolmanovsky are with the Department of Aerospace Engineering, University of Michigan, Ann Arbor, MI 48109, USA. {\tt\small \{hsiaoli, anouck, ilya\}@umich.edu}}

\thanks{$^{\dagger}$H. Eric Tseng, Retired Senior Technical Leader, Ford Research and Advanced Engineering, Chief Technologist, Excelled Tracer LLC.
{\tt\small hongtei.tseng@gmail.com}}

\thanks{This research was supported by the University of Michigan / Ford Motor Company Alliance program, and by the National Science Foundation under Awards CMMI-1904394 and ECCS-1931738.}
}

\author{Xiao Li$^{*}$, H. Eric Tseng$^{\dagger}$, Anouck Girard$^{*}$, Ilya Kolmanovsky$^{*}$}

\maketitle

\begin{abstract}
Autonomous driving depends on perception systems to understand the environment and to inform downstream decision-making. While advanced perception systems utilizing black-box Deep Neural Networks (DNNs) demonstrate human-like comprehension, their unpredictable behavior and lack of interpretability may hinder their deployment in safety critical scenarios. In this paper, we develop an Ensemble of DNN regressors (Deep Ensemble) that generates predictions with quantification of prediction uncertainties. In the scenario of Adaptive Cruise Control (ACC), we employ the Deep Ensemble to estimate distance headway to the lead vehicle from RGB images and enable the downstream controller to account for the estimation uncertainty. We develop an adaptive cruise controller that utilizes Stochastic Model Predictive Control (MPC) with chance constraints to provide a probabilistic safety guarantee. We evaluate our ACC algorithm using a high-fidelity traffic simulator and a real-world traffic dataset and demonstrate the ability of the proposed approach to effect speed tracking and car following while maintaining a safe distance headway. The out-of-distribution scenarios are also examined. 
\end{abstract}

\section{Introduction}\label{sec:intro}
Autonomous driving algorithms are typically structured as a pipeline of individual modules: The perception module gathers environmental information, while the decision-making module uses this information to make maneuver decisions. Recent advances in Deep Neural Networks (DNNs) have enabled autonomous vehicles with human-like perception capability to effectively extract information about the surroundings. For instance, research has been dedicated to integrating DNN-enable perception functions, such as localization \cite{lu2019l3} and mapping \cite{roddick2020predicting}, into autonomous driving systems. However, a significant drawback of DNN-based perceptions is the lack of interpretability; furthermore, the ability of DNNs to generalize may be limited by the coverage of the available training data. 

Furthermore, even though Neural Networks are universal function approximators \cite{hornik1989multilayer}, they have approximation errors. Moreover, in the case of Out-Of-Distribution (OOD) observations, the performance of DNNs becomes even more unpredictable, e.g., when the testing data follows a different statistical distribution from that observed during training \cite{goodfellow2016deep}. The uncertainty in perception can also impact the performance of downstream decision-making. In this paper, we consider Adaptive Cruise Control (ACC) leveraging camera sensors. The controller needs to track driver-set speed while maintaining a safe distance from the lead vehicle. In such safety-critical scenarios, accounting for the uncertainty in perception is crucial to the decision-making and control design, and vital for securing safety at the system level. 

In particular, methods have been developed in the literature to quantify DNN uncertainties. Bayesian Neural Networks \cite{neal2012bayesian} have been investigated to represent the uncertainties in DNN predictions via probabilistic modeling of the neural network parameters. Subsequent works have explored the Monte Carlo Dropout technique to reduce the computation burden in Bayesian NN \cite{gal2016dropout}. To enhance the robustness of DNNs against adversarial attacks, methods have been developed to create an ensemble utilizing a diverse set of DNNs for a single task \cite{deep_ensemble}. Additionally, other approaches, such as Laplace Approximation \cite{ritter2018scalable}, have been proposed to quantify DNN uncertainties. Control co-designs have been studied, under the assumption of bounded DNN errors, to track trajectories \cite{dean2020robust} and ensure system-level safety \cite{li2023system}. However, these approaches are limited to in-distribution settings \cite{dean2020robust}. 

In contrast, this work explores Deep Ensembles \cite{deep_ensemble} due to their good empirical performance in handling OOD scenarios. Specifically, we investigate the application to ACC using camera sensors and we develop an ensemble of DNNs to estimate the distance headway from RGB images of the lead vehicle. Subsequently, we formulate a Stochastic MPC problem to accelerate and brake the ego vehicle for speed-tracking and car-following. The algorithms we propose offer several potential advantages:
\begin{itemize}
    \item The Deep Ensemble employs a heterogeneous set of DNNs, that both generate distance headway estimation and quantify the estimation uncertainties, from RGB images of the lead vehicle.
    \item Leveraging the results from the Deep Ensemble, the Stochastic MPC is utilized for ACC, guaranteeing probabilistic safety through the integration of chance constraints. 
    \item The proposed ACC algorithm achieves good performance in car-following and speed-tracking tasks, ensuring safety in both in-distribution and OOD scenarios, as verified using a high-fidelity simulator. 
\end{itemize}

This paper is organized as follows: In Sec.~\ref{sec:problem}, we introduce the ACC problem. We also outline the assumptions made regarding vehicle kinematics and the ACC design objectives. In Sec.~\ref{sec:method}, we present our Deep Ensemble development that estimates the distance headway which informs the subsequent Stochastic MPC to control the acceleration of the ego vehicle. In Sec.~\ref{sec:results}, we demonstrate the Deep Ensemble's ability to provide estimations and quantify estimation uncertainties. Furthermore, we validate the proposed adaptive cruise controller, using a high-fidelity simulator and a real-world traffic dataset. Finally, Sec.~\ref{sec:conclusion} provides conclusions.

\section{Problem Formulation}\label{sec:problem}
\begin{figure}[!h]
    \centering
    \includegraphics[width=0.47\textwidth]{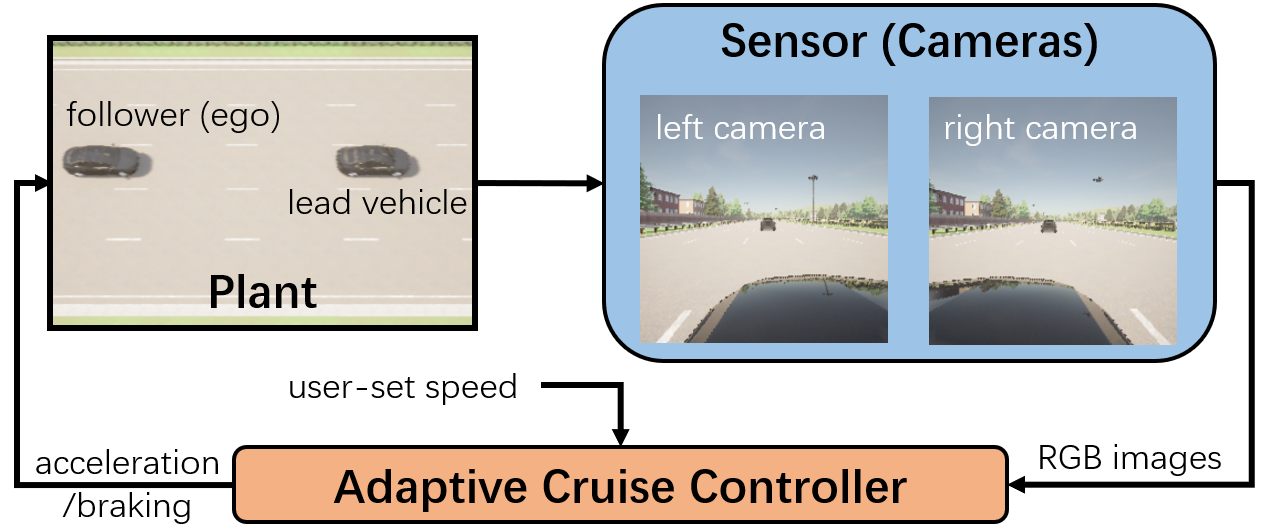}
    \caption{A schematic diagram of the Adaptive Cruise Control (ACC) scenario: The follower (ego vehicle) keeps a safe distance headway to the lead vehicle in the front leveraging camera sensors.}
    \label{fig:acc_problem}
\end{figure}

In this paper, we focus on control design with visual perception (i.e., cameras) in the loop. As shown in Fig.~\ref{fig:acc_problem}, the ego vehicle observes the lead vehicle in the front via camera sensors, installed to the left and right of the ego vehicle's front window. Using the RGB images from the cameras, the ego vehicle estimates its distance headway to the lead vehicle and, subsequently, commands its acceleration and braking to keep a safe distance headway and track a desired speed. 

We use the following discrete-time model to represent the vehicle kinematics,
\begin{equation}\label{eq:kinematics}
\begin{aligned}    
    x_{k+1} &= x_k + v_k \Delta t + \frac{1}{2}a_k \Delta t^2, \\
    v_{k+1} &= v_k + a_k \Delta t,
\end{aligned}
\end{equation}
where $x_k$, $v_k$, and $a_k$ are the longitudinal position, velocity, and acceleration at time instance $t_k$, respectively; $\Delta t > 0$ is the time in second elapsed between discrete time instances $t_k$ and $t_{k+1}$. Since we focus on car-following development, we only consider the longitudinal kinematics in Eq.~\eqref{eq:kinematics} while neglecting the lateral ones. Namely, we assume both ego and lead vehicles follow this dynamics model, do longitudinal acceleration or braking maneuvers, and keep the current lane. In the sequel, we use variables with superscripts $x_k^{(l)}$, $v_k^{(l)}$ to represent the states of the lead vehicle while those without (e.g., $x_k$, $v_k$, and $a_k$) denote the states of the ego vehicle. We use the following models to represent the camera sensor measurements,
\begin{equation}\label{eq:sensor}
    I_{k,l} = q_l(x_k, x_k^{(l)}),\; I_{k,r} = q_r(x_k, x_k^{(l)}),
\end{equation}
where the measurements $I_{k,l},I_{k,r}\in\R^{3\times 224\times 224}$ are RGB images of $3$ color channels and size $224\times 224$ acquired from the left and right cameras, respectively. 

In this work, we consider the development of an adaptive cruise controller that adopts the following form 
\begin{equation}\label{eq:controller}
    a_k = K(I_{k,l},I_{k,r}, v_s),
\end{equation}
where $v_s$ is the driver-set ACC speed. The controller $K$ computes the acceleration/deceleration command for the ego vehicle based on a pair of RGB images while incorporating the following control objectives and constraints:
\begin{itemize}
    \item safety: keep an adequate distance headway $d_k$, defined as the vehicle bumper-to-bumper distance, to the lead vehicle to prevent potential collisions.
    \item fuel economy: minimize the accumulated acceleration effort $\sum_{i=0}^{N-1} \abs{a_{k+i}}$ over a horizon of length $N$.
    \item driving comfort: minimize the rate of change in the acceleration trajectory $(a_{k+i})_{i=0}^{N-1}$.
    \item speed tracking: track the driver-set speed $v_s$.
    \item speed and acceleration limits: the speed $v_k$ and the acceleration $a_k$ within the interval $[v_{\min}, v_{\max}]$ and $[a_{\min}, a_{\max}]$, respectively.
\end{itemize}
This problem is challenging due to the high dimensionality of the image space, which can induce unpredictable behavior of the controller and, subsequently, raise safety concerns.

\section{Method}\label{sec:method}
\begin{figure*}[ht!]
    \vspace{0.5em}
    \centering
    \begin{subfigure}[t]{0.38\textwidth}
        \centering
        \includegraphics[width=\textwidth]{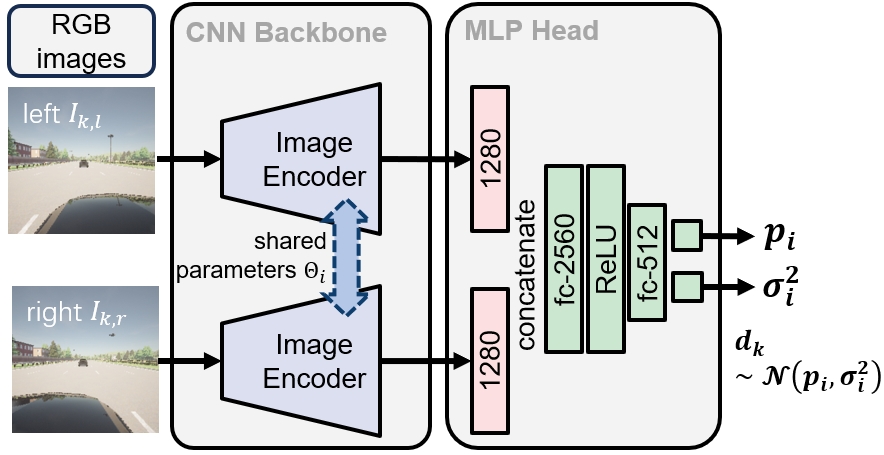}
        \caption{}
        \label{fig:dnn_arch}
    \end{subfigure}
    ~
    \begin{subfigure}[t]{0.59\textwidth}
        \centering
        \includegraphics[width=\textwidth]{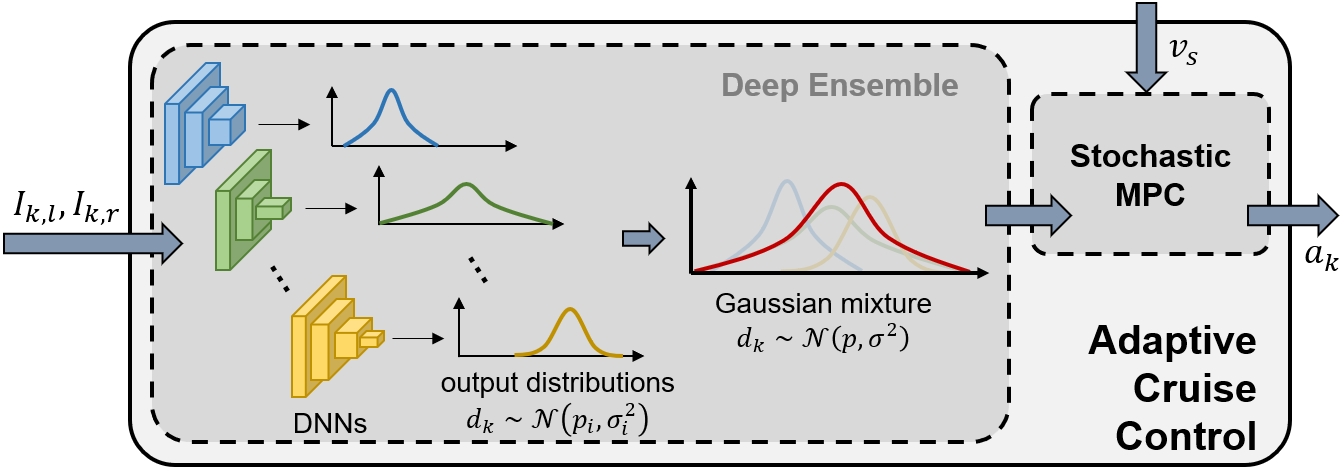}
        \caption{}
        \label{fig:ensemble_acc}
    \end{subfigure}
    \vspace{-0.5em}
    \caption{Schematic diagrams of adaptive cruise controller design. (a) Each DNN differs in the CNN architecture of the image encoder, and estimates the distribution of the distance headway from input RGB images. (b) An ensemble of DNNs, with a heterogeneous set of CNN architectures as image encoders, collectively estimates the distance headway as a Gaussian mixture; then a Stochastic MPC uses estimated headway mean and variance to compute the acceleration/deceleration command for the ego vehicle.}
    \label{fig:method}
\end{figure*}

We propose a modularized ACC development approach to enhance the safety guarantees. As shown in Fig.~\ref{fig:ensemble_acc}, a Deep Ensemble estimates the distance headway $d_k$ as a Gaussian distribution and the variance quantifies the estimation uncertainty in Sec.~\ref{subsec:ensemble}. Subsequently, a Stochastic MPC is utilized to optimize the acceleration trajectory to realize the aforementioned design objectives while ensuring probabilistic safety in Sec.~\ref{subsec:acc}. 

\subsection{Deep Neural Network Ensemble}\label{subsec:ensemble}

We implement a Deep Ensemble to estimate the distance headway $d_k$ ($d_k\in\R$, $d_k\geq 0$) to the lead vehicle given a pair of RGB images $I_{k,l}$, $I_{k,r}$ from on-board cameras. In this regression problem, the estimates from a regressor admit the following form,
\begin{equation}\label{eq:error}
     d_k = p(I_{k,l},I_{k,r}) + e(I_{k,l},I_{k,r}),
\end{equation}
where $p(I_{k,l},I_{k,r})$ is the distance headway estimate and $e(I_{k,l},I_{k,r})$ is the estimation error that depends on the current image observations. Typical approaches in the literature focus on learning an accurate mapping $p(I_{k,l},I_{k,r})$ that minimizes $\abs{d_k - p(I_{k,l},I_{k,r})}$ and do not pursue characterizing the behavior of the error $ e(I_{k,l},I_{k,r})$. The high dimensionality of image space requires complex Convolution Neural Networks (CNNs) as image encoders, which results in more unpredictable error dynamics $ e(I_{k,l},I_{k,r})$; hence in our work we are focusing on further modeling and characterizing this error. 

Inspired by \cite{est_distribution}, we assume the error $e(I_{k,l},I_{k,r})$ is zero-mean Gaussian. We develop DNNs that can simultaneously generate an estimate of $p(I_{k,l},I_{k,r})$ and quantify the estimation uncertainties by predicting the variance of the error $e(I_{k,l},I_{k,r})$. As shown in Fig.~\ref{fig:dnn_arch}, the individual $i$th DNN comprises two blocks: the CNN backbone takes two images $I_{k,l},I_{k,r}$ and embeds them into vectors $z_{k,l},z_{k,r}\in \R^{1280}$ using two identical CNN image encoders with shared parameters (i.e., $z_{k,l}=f_{\text{cnn}}(I_{k,l}|\Theta_i)$, $z_{k,r}=f_{\text{cnn}}(I_{k,r}|\Theta_i)$, and $\Theta_i$ is the shared parameters in the CNN image encoders); the subsequent Multi-Layer Perceptron (MLP) computes two outputs $p_i,\sigma^2_i$ from input vectors $z_{k,l},\;z_{k,r}$ according to,
\begin{equation}\label{eq:mlp}
\begin{array}{c}
     z_0 = [z_{k,l}^T ~~ z_{k,r}^T]^T,~~ z_1 = \sigma_{\text{ReLU}}\left(W_{i,1}z_0+b_{i,1}\right),\\
     z_2 = W_{i,2} z_1+b_{i,2}, ~~ [p_i ~~ \sigma^2_i]^T = z_2,
\end{array}
\end{equation}
where  $W_{i,1}\in\R^{2560\times 512}, W_{i,2}\in\R^{512\times 2}$ and $b_{i,1}\in\R^{512}, b_{i,2}\in\R^{2}$ are the network parameters; $\sigma_{\text{ReLU}}(z)=\max\set{0,z}$ is an element-wise ReLU activation function. The $i$th DNN produces estimate $p_i(I_{k,l},I_{k,r})$ of the actual distance headway $d_k$. It also computes an estimate of the variance $\sigma_i^2(I_{k,l},I_{k,r})$ of the error $e(I_{k,l},I_{k,r})$.

Given a training trajectory $\mathcal{D} = (d_k, I_{k,l},I_{k,r})_{k=1}^{M}$, the parameter $\Theta$ of the $i$th DNN, i.e., $\Theta = \set{\Theta_i, W_{i,1},W_{i,2},b_{i,1},b_{i,2}}$, is optimized using the following proposition:
\begin{prop}\label{prop:loss}
Given a training trajectory $\mathcal{D} = (d_k, I_{k,l},I_{k,r})_{k=1}^{M}$, assuming each data point $(d_k, I_{k,l},I_{k,r})\in D$ is independently collected, and the error is zero-mean Gaussian, i.e., $e(I_{k,l},I_{k,r})\sim\mathcal{N}\brk{0, \sigma_i^2(I_{k,l},I_{k,r})}$, the optimal parameter is attained according to the following likelihood maximization,
\begin{equation}\label{eq:likelihood}
    \Theta^* = \argmax\limits_{\Theta} \prob{\mathcal{D} | \Theta},
\end{equation}
and it is equivalent to the following optimization,
\begin{equation}\label{eq:loss}
\begin{array}{c}
    \Theta^* = \argmin\limits_{\Theta} \mathcal{L}\brk{\mathcal{D}|\Theta}  = \argmin\limits_{\Theta}  \sum_{k=1}^M \\
    \left[
    \log{\sigma_i^2(I_{k,l},I_{k,r}|\Theta)} + \frac{\brk{d_k - p_i(I_{k,l},I_{k,r}|\Theta)}^2}{\sigma_i^2(I_{k,l},I_{k,r}|\Theta)}
    \right].
\end{array}
\end{equation}
\end{prop}
In the case of a large dataset $\mathcal{D}$, we note that an iterative training algorithm based on Monte Carlo Sampling can be applied, i.e., batch Stochastic Gradient Descent (SGD), where a mini-batch dataset $\mathcal{D'}\subset \mathcal{D}$ is sampled to update the parameter $\Theta^*$ according to Eq.~\eqref{eq:loss} at each iteration. The proof is presented as follows:
\begin{proof}
The likelihood in \eqref{eq:loss} can be rewritten according to
\begin{equation*}
\begin{array}{l}
     \prob{\mathcal{D} | \Theta} = \prod_{k=1}^M \prob{d_k|I_{k,l},I_{k,r},\Theta}\\
     = \prod_{k=1}^M\frac{1}{\sigma_i(I_{k,l},I_{k,r}|\Theta)\sqrt{2\pi}}\exp{{-\frac{1}{2}\brk{\frac{d_k-p_i(I_{k,l},I_{k,r}|\Theta)}{\sigma_i(I_{k,l},I_{k,r}|\Theta)}}^2}}
\end{array}    
\end{equation*}
where the first equality is derived from the independence assumption, and the second equality is due to the zero mean Gaussian assumption of $e(I_{k,l},I_{k,r})$. Then, the maximization in Eq.~\eqref{eq:likelihood} is equivalent to the following minimization, $\argmin\limits_{\Theta} \big(-\log{ \prob{\mathcal{D} | \Theta} }\big)$, where this minimization of the negative log-likelihood is equivalent to that in Eq.~\eqref{eq:loss}.
\end{proof}

Furthermore, we adopt the idea of Deep Ensemble~\cite{deep_ensemble} to improve the robustness of the distance headway estimation in OOD scenarios. The Deep Ensemble comprises $n$ different DNNs of various CNN architectures as the image encoders (see Fig.~\ref{fig:method}). Individually, the $i$th DNN in the Deep Ensemble is trained to generate predictions $p_i(I_{k,l},I_{k,r})$, $ \sigma_i^2(I_{k,l},I_{k,r})$, where the actual distance headway $d_k$ follows a Gaussian distribution $\mathcal{N}\big(p_i(I_{k,l},I_{k,r}), \sigma_i^2(I_{k,l},I_{k,r})\big)$ according to assumptions in  Proposition~\ref{prop:loss}. Collectively, $n$ DNNs in the Deep Ensemble form a Gaussian mixture, and produce the final distance headway estimates according to, 
\begin{equation}\label{eq:mixture}
\begin{aligned}
    p_k &= \frac{1}{n}\sum_{i=1}^n p_i(I_{k,l},I_{k,r}),
    \\
    \sigma^2_k &= \frac{1}{n}\sum_{i=1}^n\big(\sigma_i^2(I_{k,l},I_{k,r}) + p^2_i(I_{k,l},I_{k,r})\big) - p^2_k,
\end{aligned}
\end{equation}
where here and in the sequel we drop the dependence of $p_k,\sigma^2_k$ on $(I_{k,l},I_{k,r})$ to simplify the notations. Eventually, the actual distance headway follows a Gaussian distribution derived from the Gaussian mixture, i.e., $d_k\sim \mathcal{N}(p_k, \sigma^2_k)$. 

\subsection{Adaptive Cruise Control}\label{subsec:acc}
At the current time $t_k$, we assume that previous acceleration $a_{k-1}$ and the distance headway estimates $p_{k-1}, \sigma^2_{k-1}, p_k, \sigma^2_k$ generated from the Deep Ensemble are known. Note that the actual distance headway, i.e., $d_{k-1}$ and $d_k$, is unknown to the algorithm, but the following results hold, $d_k\sim \mathcal{N}(p_k, \sigma^2_k)$, $d_{k-1}\sim \mathcal{N}(p_{k-1}, \sigma^2_{k-1})$. Then, we can predict the distributions of future distance headway for a variable acceleration trajectory using the following proposition:
\begin{prop}\label{prop:gaussian}
Given $a_{k-1}$, an acceleration trajectory $(a_{k+i})_{i=0}^{N-1}$ of length $N$, and distribution parameters $p_{k-1}$, $\sigma^2_{k-1}$, $p_k$, $\sigma^2_k$, such that the unknown distance headway obeys $d_k\sim \mathcal{N}(p_k, \sigma^2_k)$, $d_{k-1}\sim \mathcal{N}(p_{k-1}, \sigma^2_{k-1})$, and if the lead vehicle has a constant speed, then, the variables $d_{k+i}$, $\Delta v_{k+i}$, $i=0,\dots,N$ are Gaussian distributed,
\begin{equation}\label{eq:all_gaussian}
\begin{array}{c}
d_{k+i}\sim\mathcal{N}(p_{k+i}, \sigma^2_{k+i}), 
\\
\Delta v_{k+i} \sim\mathcal{N}(p'_{k+i}, \sigma'^2_{k+i}), ~~ i=0,\dots,N,
\end{array}
\end{equation}
where $\Delta v_{k+i} = v_{k+i}^{(l)} - v_{k+i}$ is the speed difference between the lead and ego vehicles. Furthermore, the distribution parameters $p_{k+i}, \sigma^2_{k+i}$ and $p'_{k+i}, \sigma'^2_{k+i}$ can be recursively derived using the following results,
\begin{equation*}
\begin{array}{c}
    p'_k=\frac{1}{\Delta t}(p_k-p_{k-1}) -\frac{1}{2}a_{k-1}\Delta t, 
    ~
    \sigma'^2_k = \frac{1}{\Delta t^2} (\sigma^2_k+ \sigma^2_{k-1}),
    \\
    p_{k+i+1} = p_{k+i}+ p'_{k+i}\Delta t -\frac{1}{2}a_{k+i}\Delta t^2, 
    \\
    \sigma^2_{k+i+1} = \sigma^2_{k+i} + \Delta t^2\sigma'^2_{k+i},
    ~
    p'_{k+i+1} = p'_{k+i}-a_{k+i}\Delta t, 
    \\
    \sigma'^2_{k+i+1} = \frac{2}{\Delta t^2}\sigma^2_{k+i} + \sigma'^2_{k+i},
    ~
    i=0,\dots,N-1,
\end{array}
\end{equation*}
where the distribution means $p_{k+i}, p'_{k+i}$ linearly depend on the variables $(a_{k+i})_{i=0}^{N-1}$, and variances $\sigma^2_{k+i},\sigma'^2_{k+i}$ are constants, for all $i=0,\dots,N$. 
\end{prop} 

\begin{proof}
Assuming the lead vehicle maintains a constant speed, the proposition above can be derived from the following equalities, $\Delta v_k=\frac{1}{\Delta t}(d_k-d_{k-1})-\frac{1}{2}a_{k-1}\Delta t$, $d_{k+i+1} = d_{k+i}+\Delta v_{k+i}\Delta t -\frac{1}{2}a_{k+i}\Delta t^2$, $\Delta v_{k+i+1} =\frac{1}{\Delta t}(d_{k+i+1}-d_{k+i})-\frac{1}{2}a_{k+i}\Delta t$, $i=0,\dots,N-1$. Based on Propostion~\ref{prop:gaussian}, we establish the prediction of future distance headway $d_{k+i}$ and speed difference $\Delta v_{k+i}$ as Gaussian distributions with the means being the linear functions of the acceleration trajectory $(a_{k+i})_{i=0}^{N-1}$ and constant variances. 
\end{proof}

Hence, treating $(a_{k+i})_{i=0}^{N-1}$ as decision variables, we formulate a Stochastic MPC problem that predicts the distributions of the future distance headway and speed difference for different $(a_{k+i})_{i=0}^{N-1}$ and optimizes  $(a_{k+i})_{i=0}^{N-1}$ while incorporating the objectives in Sec.~\ref{sec:problem} according to,
\begin{subequations}\label{eq:SMPC}
\begin{multline}\label{eq:SMPC_obj}
    \argmin\limits_{\substack{a_{k+i-1}, v_{k+i}, p_{k+i},\\ p'_{k+i},~i=1,\dots,N}}
    ~
    \E
    \Big[
    \sum_{i=0}^{N-1} r_1 a^2_{k+i} + r_2 (a_{k+i}-a_{k+i-1})^2 
    \\
     + \sum_{i=1}^N q_1(v_{k+i}-v_s)^2 + q_2 \Delta v^2_{k+i}
    \Big]
\end{multline}
subject to:
\begin{align}
    \prob{d_{k+i}\geq d_s+T_sv_{k+i}}\geq 1-\epsilon_i,\label{eq:SMPC_chance1}
    \\
    d_{k+i}\sim\mathcal{N}(p_{k+i}, \sigma^2_{k+i}), ~\Delta v_{k+i}\sim\mathcal{N}(p'_{k+i}, \sigma'^2_{k+i}),\label{eq:SMPC_chance2}
    \\
    p_{k+i} = p_{k+i-1}+ p'_{k+i-1}\Delta t -\frac{1}{2}a_{k+i-1}\Delta t^2, \label{eq:SMPC_eq1}
    \\
    p'_{k+i} = p'_{k+i-1}-a_{k+i-1}\Delta t, \label{eq:SMPC_eq2}
    \\
    v_{\min}\leq v_{k+i}\leq v_{\max}, ~ a_{\min}\leq a_{k+i-1}\leq a_{\max},\label{eq:SMPC_limits}
    \\
    v_{k+i}=v_{k+i-1}+a_{k+i-1}\Delta t,~i=1,\dots,N,\label{eq:SMPC_eom}
\end{align}
\end{subequations}
where $N$ is the prediction horizon; $\epsilon_i\in(0,1]$, $i=1,\dots,N$ are tunable positive constants; $v_s$ is the driver-set target speed; $d_s$, $T_s$ are the adjustable ACC stopping distance, and constant time headway, respectively. Meanwhile, Proposition~\ref{prop:gaussian} implies larger variances $\sigma^2_{k+i}$, $\sigma'^2_{k+i}$ for prediction horizon $i$ further in the future. We set the constants $\epsilon_i$ to satisfy the following inequality, $\epsilon_1\leq \cdots\leq \epsilon_N$, such that the chance constraints~\eqref{eq:SMPC_chance1} are relaxed more further along the horizon. The variables $r_1,r_2, q_1,q_2$ are tunable weights that balance the minimization of control effort, the reduction of the rate of changes in control, speed tracking, and lead vehicle-following, respectively. Furthermore, we use the chance constraints in Eq.~\eqref{eq:SMPC_chance1}, \eqref{eq:SMPC_chance2} to enforce a sufficient distance headway in probability, and the equalities \eqref{eq:SMPC_eq1}, \eqref{eq:SMPC_eq2} propagate the distribution means based on the results from Proposition~\ref{prop:gaussian}.

Note that the Gaussian-distributed random variables $d_{k+i}$, $\Delta v_{k+i}$ and chance constraints render the MPC problem~\eqref{eq:SMPC} stochastic. To make the problem machine solvable, we transcript the Stochastic MPC into a deterministic one using the following result:
\begin{prop}\label{prop:DMPC}
Under the condition that $\delta_i=0$ for all $i=1,\dots,N$, solving the following \textbf{Quadratic Programming} problem recovers the solution of the Stochastic MPC problem \eqref{eq:SMPC},
\begin{subequations}\label{eq:DMPC}
\begin{multline}\label{eq:DMPC_obj}
    \argmin\limits_{\substack{a_{k+i-1}, v_{k+i}, p_{k+i},\\ \delta_i, p'_{k+i},~i=1,\dots,N}}
    ~
    \sum_{i=0}^{N-1} r_1 a^2_{k+i} + r_2 (a_{k+i}-a_{k+i-1})^2 
    \\
     + \sum_{i=1}^N q_1(v_{k+i}-v_s)^2 + q_2 p'^2_{k+i}+\rho \delta_i
\end{multline}
subject to:
\begin{align}
    p_{k+i}\geq d_s+T_s v_{k+i} + \sqrt{2\sigma^2_{k+i}}{\tt erf}^{-1}(1-2\epsilon_i) -\delta_i\label{eq:DMPC_safe}
    \\    \eqref{eq:SMPC_eq1},~\eqref{eq:SMPC_eq2},~\eqref{eq:SMPC_limits},~\eqref{eq:SMPC_eom},~\delta_i\geq0,~i=1,\dots,N,
\end{align}
\end{subequations}
where ${\tt erf}^{-1}$ is the inverse image of the Gauss error function, non-negative variables $\delta_i$, $i=1,\dots,N$ are used to relax the constraints~\eqref{eq:DMPC_safe} ensuring recursive feasibility and the weight $\rho$ is adjusted to penalize the violations of soft constraints~\eqref{eq:DMPC_safe} due to the introduction of $\delta_i$.
\end{prop}

\begin{proof}
We consider the condition that $\delta_i=0$ for all $i=1,\dots,N$. The only random variables in Eq.~\eqref{eq:SMPC_obj} are $\Delta v_{k+i}\sim\mathcal{N}(p'_{k+i}, \sigma'^2_{k+i})$, therefore, the other terms can be moved out of the expectation. Moreover, we can establish the following equality, $\E\bsq{\Delta v_{k+i}^2}=\sigma'^2_{k+i} + \E\bsq{\Delta v_{k+i}}^2=\sigma'^2_{k+i} + p'^2_{k+i}$, which combined with $\sigma'^2_{k+i}$ being a constant from the results in Proposition~\ref{prop:gaussian} proves that $\argmin \sum_{i=1}^N \E \bsq{q_2 \Delta v^2_{k+i}}= \argmin \sum_{i=1}^N q_2 p'^2_{k+i}$. Namely, the optimization objectives of \eqref{eq:SMPC} and \eqref{eq:DMPC} are equivalent. Furthermore, the chance constraints~\eqref{eq:SMPC_chance1},~\eqref{eq:SMPC_chance2} are equivalent to Eq.~\eqref{eq:DMPC_safe}.
\end{proof}

Eventually, provided with the estimated distributions from Deep Ensemble, we formulate a Deterministic MPC that is recursively feasible, predicts the future distance headway distributions, and computes an acceleration trajectory $(a_{k+i})_{i=0}^{N-1}$ ensuring probabilistic safety. 
\section{Case Studies}\label{sec:results}
Here, we demonstrate the effectiveness of the proposed ACC algorithm. The Deep Ensemble is trained and evaluated using a high-fidelity simulation environment in Sec.~\ref{subsec:result_ensemble}. We showcase the proposed ACC algorithm that integrates the Deep Ensemble with the Stochastic MPC in a simulation example in Sec.~\ref{subsec:result_acc}, and report the quantitative results in Sec.~\ref{subsec:result_acc_batch} leveraging a high-fidelity simulator and real-world vehicle trajectories. Finally, in comparison with the in-distribution example provided in Sec.~\ref{subsec:result_acc}, the performance of the proposed algorithm is demonstrated in OOD scenarios in Sec.~\ref{subsec:result_ood}.
\subsection{Deep Ensemble for Distance Headway Estimation}\label{subsec:result_ensemble}

\begin{figure}[!h]
    \centering
    \includegraphics[width=0.45\textwidth]{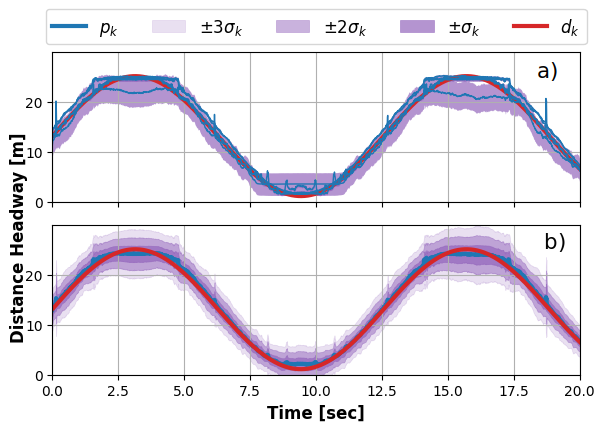}
    \caption{Distance headway estimates $p_k$ (blue lines) with uncertainty quantification using 1$\sigma$, 2$\sigma$, and 3$\sigma$ intervals (purple bands) versus the actual $d_k$ (red lines). (a) The results, $p_i(I_{k,l},I_{k,r}), \sigma_i^2(I_{k,l},I_{k,r})$, $i=1,\dots,6$, of each DNN visualized in overlap. (b) Deep Ensemble estimation results. A demonstration video is available in \url{https://bit.ly/3TCM5lC}.}
    \label{fig:nn_testing}
\end{figure}

In the Deep Ensemble, we integrate 6 DNNs, i.e., $n=6$, with each of them employing a different CNN architecture as the image encoder. We utilize the following CNNs due to their outstanding performance as image encoders in solving image classification problems: ResNet50~\cite{resnet}, GoogleNet~\cite{googlenet}, AlexNet~\cite{alexnet}, MobileNetV2~\cite{mobilenet}, EfficientNet~\cite{efficientnet}, and VGG16~\cite{vgg}. The goal is to train individual DNNs to predict the distribution parameters $p_i(I_{k,l},I_{k,r})$, $\sigma_i^2(I_{k,l},I_{k,r})$ given the corresponding RGB images $I_{k,l}, I_{k,r}$, such that $d_k\sim\mathcal{N}\big(p_i(I_{k,l},I_{k,r}), \sigma_i^2(I_{k,l},I_{k,r})\big)$. 

We use the Carla simulator~\cite{carla} to collect datasets and test our developments. We collect a dataset $D=(d_k, I_{k,l}, I_{k,r})_{k=1}^{20706}$ of 20706 data triplets. The data points are collected in the map {\tt Town06} in Carla. To simplify the exposition of the approach, we fix the model of the lead vehicle to {\tt vehicle.lincoln.mkz\_2020} (2020 Lincoln MKZ Sedan) and set the weather to {\tt ClearNoon} (good lighting conditions, no rain, and no objects casting shadow). The distance headway in the dataset ranges from 1 to 25 m, where data points with a distance headway larger than 25 m are neglected due to resolution limitations of the cameras. 

We use Python with Pytorch~\cite{pytorch} to train the DNN using Batch SGD and the loss function~\eqref{eq:loss} defined among the mini-batch dataset. We train the DNNs for 100 epochs using a Batch SGD with momentum. The mini-batch sizes are set to 60, 105, 500, 65, 60, and 75 (maximum batch size capability of a Nvidia GeForce RTX 4080 GPU with 16 GB memory) for ResNet50, GoogleNet, AlexNet, MobileNetV2, EfficientNet, and VGG16, respectively. We set the learning rate and momentum to $0.001$ and $0.9$, respectively. Each time we train a different DNN, we randomly select 80\% data points for training and 20\% data points for validation, while the DNN initial parameters $\Theta$ are also randomly initialized. All images input to the DNN are normalized using the following function in the {\tt torchvision} package,
\begin{multline*}    
    \tt transforms.Normalize\\
    ([0.485, 0.456, 0.406], [0.229, 0.224, 0.225]).
\end{multline*}
Meanwhile, to ensure numerical stability with the logarithm in the loss function~\eqref{eq:loss}, we enforce the positiveness of the output variance $\sigma_i^2(I_{k,l},I_{k,r})$ without significantly altering its value, using the following assignment,
\begin{equation*}
    \sigma_i^2(I_{k,l},I_{k,r}) \leftarrow \epsilon + \log\brk{1+\exp{\sigma_i^2(I_{k,l},I_{k,r})}},
\end{equation*}
where a small $\epsilon>0$ is chosen, e.g., $\epsilon=10^{-6}$. 

To evaluate the performance of each DNN and the Deep Ensemble, we separately collect a testing trajectory $(d_k, I_{k,l}, I_{k,r})_{k=1}^{M'}$ of 20 seconds, and the results are reported in Fig.~\ref{fig:nn_testing}. The 1$\sigma$ band in the Deep Ensemble results, i.e., $p_k\pm \sigma_k$, contains the actual distance headway $d_k$ which demonstrates the effectiveness of our method in both providing accurate estimates and quantifying the estimation uncertainties. Moreover, we also note that results from individual DNNs differ from each other significantly when the distance headway is large. This is due to the resolution limitation of the RGB images, where the lead vehicle vanishes as a black pixel when the distance headway is larger than 20 m. However, the Deep Ensemble can reflect this uncertainty using a larger variance in the estimation results.

\subsection{Adaptive Cruise Control in Carla Simulation}\label{subsec:result_acc}
We construct car-following scenarios using the Carla simulator~\cite{carla}, where the follower ego vehicle is controlled by our algorithm to follow a lead vehicle. In the sequel, we leverage a naturalistic traffic trajectory dataset, named High-D dataset~\cite{highD}, to configure realistic car-following scenarios. The High-D dataset records real-world vehicle trajectories in German freeways. The statistics visualized in Fig.~\ref{fig:highD_stats} are obtained from data of 110,500 vehicles driven over 44,500 kilometers. 

\begin{figure}[!h]
    \centering
    \vspace{0.5em}
    \includegraphics[width=0.45\textwidth]{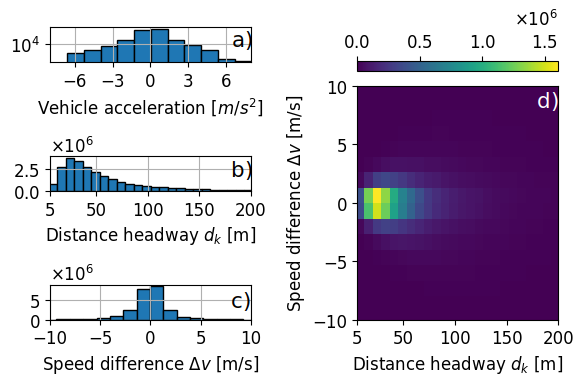}
    \caption{Histogram of vehicle driving statistics in High-D dataset: (a) longitudinal acceleration/deceleration (y-axis in log scale); (b) distance headway and (c) speed difference between lead and follower vehicles; (d) 2D histogram combining statistics in (b) and (c).}
    \label{fig:highD_stats}
\end{figure}

The lower speed limit is set to $v_{\min}=0\;\rm m/s$ while the upper speed limit of the dataset is $v_{\max}=34\;\rm m/s$. As shown in Fig.~\ref{fig:highD_stats}, the majority of longitudinal accelerations and decelerations of High-D vehicles are within the range of $[a_{\min},a_{\max}]=[-6, 6]\;\rm m/s^2$. The Stochastic MPC has a prediction horizon of 3 sec, i.e., $N=3$ and $\Delta t=1$ sec. Furthermore, the Stochastic MPC operates in an asynchronous updating scheme and recomputes $a_k$ every 0.5 sec. Other parameters are set using the following values: $d_s=15$ m, $T_s=0$ sec, $[r_1,r_2,q_1,q_2,\rho]=[1,5,5,1,50]$, and $\epsilon_{1,2,3}=0.2,\; 0.4,\; 0.6$, respectively. Provided with distributions $d_{k+i}\sim\mathcal{N}(p_{k+i}, \sigma^2_{k+i})$, $\Delta v_{k+i}\sim\mathcal{N}(p'_{k+i}, \sigma'^2_{k+i}),$ from the Deep Ensemble, the MPC problem~\eqref{eq:DMPC} is solved using PyDrake~\cite{drake}.

\begin{figure}[!h]
    \centering
    \vspace{0.5em}
    \includegraphics[width=0.45\textwidth]{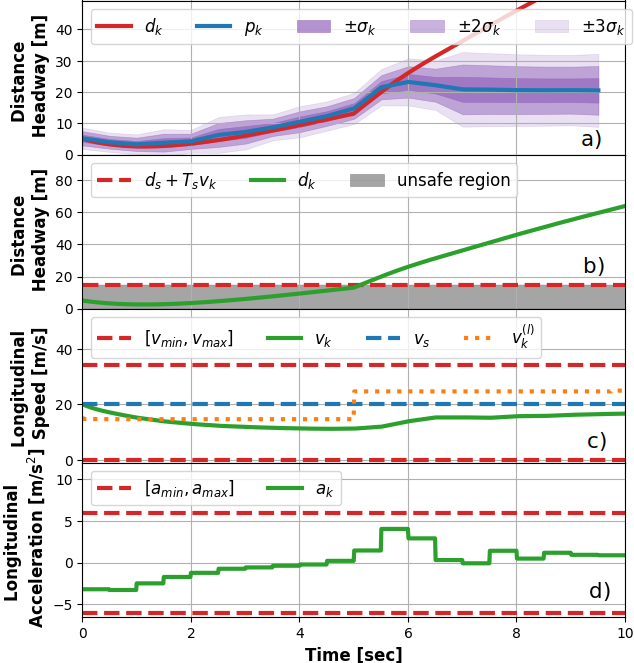}
    \caption{ACC simulation example in which the ego vehicle first follows the lead vehicle given its speed smaller than $v_s=20\;\rm m/s$, then, enters the speed-tracking mode after the step change in the lead vehicle speed: (a) distance headway estimation from the Deep Ensemble; (b) ACC algorithm regulating the ego out of the unsafe region; (c) speed trajectories of the lead and ego vehicles; (d) acceleration commands. The animation is available in \url{https://bit.ly/3TFgxLZ}.}
    \label{fig:acc_sim}
\end{figure}

A simulation example is presented in Fig.~\ref{fig:acc_sim}. Due to the limitation in the image resolution, the lead vehicle vanishes in the RGB images (see the video in \url{https://bit.ly/3TFgxLZ}) as a black pixel when $d_k\geq 20\;\rm m$. We observe that the Deep Ensemble captures this source of uncertainties by presenting amplified variances in distance headway estimations (see Fig.~\ref{fig:acc_sim}a). We also note that the occurrence of uncertain distance headway estimates when $d_k\geq 20\;\rm m$ will not affect the performance of the ACC algorithm in securing safety in the near future. Our ACC algorithm successfully decelerates the ego vehicle to keep a safe distance headway (see Fig.~\ref{fig:acc_sim}b). Moreover, as shown in Fig.~\ref{fig:acc_sim}c, with a moderate control effort, our algorithm performs car-following when the lead vehicle is at a speed lower than $v_s=20\;\rm m/s$, and tracks the driver-set speed $v_s$ when the lead vehicle accelerates to a higher speed. 

\subsection{Adaptive Cruise Control with Real-World Trajectories}\label{subsec:result_acc_batch}

\begin{figure}[!h]
    \centering
    \includegraphics[width=0.47\textwidth]{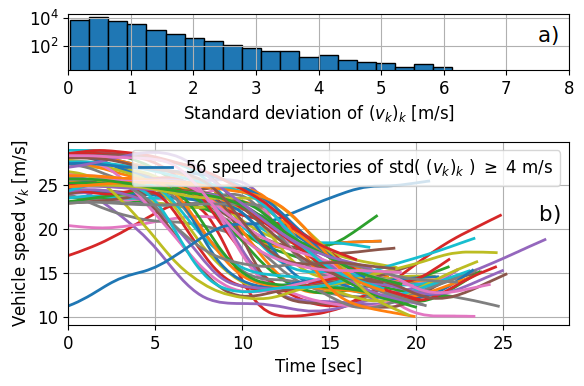}
    \caption{Sampling of lead vehicles' speed trajectories from the High-D dataset: (a) histogram of the standard deviation ${\tt std}\big((v_k)_k\big)$ of the High-D vehicle speed trajectory $(v_k)_k$ (y-axis in log scale); (b) 56 speed trajectories $(v_k)_k$ are chosen with their standard deviation larger than 4 $\rm m/s$.}
    \label{fig:highD_vTraj}
\end{figure}

We inherit parameters and ACC configurations from the previous section and use the High-D dataset~\cite{highD} to quantitatively evaluate the performance of our algorithm. We construct 56 car-following test cases using the Carla simulator~\cite{carla}. In the 56 test cases, the speed trajectories $(v_k)_{k=1}^K$ of the lead vehicles are sampled from the High-D dataset (see Fig.~\ref{fig:highD_vTraj}) with standard deviations larger than 4 $\rm m/s$. We remove speed trajectories where the lead vehicles take fewer acceleration/braking actions to be able to test our algorithm in more challenging but realistic cases. In the sequel, we use $(v_k)_k$ to denote $(v_k)_{k=1}^K$ given the length of the speed trajectory $K$ is variable.

\begin{figure}[!h]
    \centering
    \vspace{0.5em}
    \includegraphics[width=0.47\textwidth]{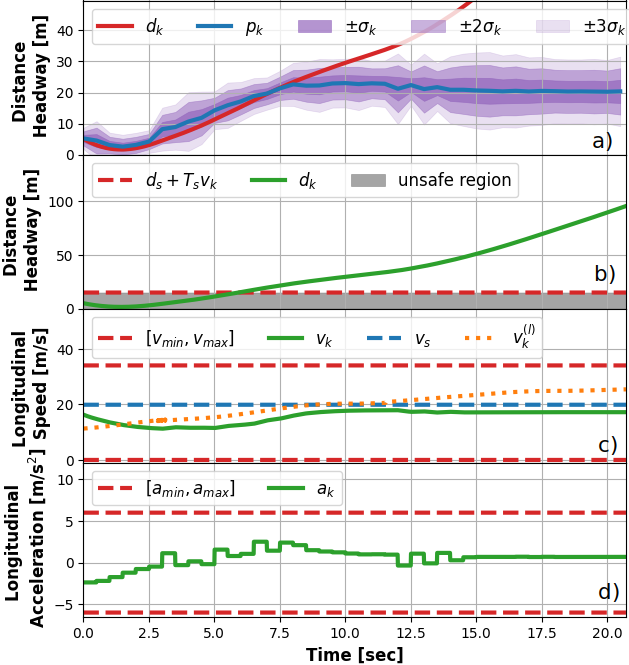}
    \caption{Another ACC testing example: the ego vehicle first follows the lead vehicle, then, tracks the set speed while keeping a safety distance headway. The animation is available in \url{https://bit.ly/4adwqP2}.}
    \label{fig:highD_example}
\end{figure}

Moreover, as shown in Fig.~\ref{fig:highD_stats}, the majority of the follower vehicles in the High-D dataset have a distance headway larger than $5\;\rm m$ and a relative speed difference within an interval of $[-5, 5]\;\rm m/s$. Hence, we initialize the follower vehicle with an initial distance headway $d_k = 5\;\rm m$ and an initial speed $5\;\rm m/s$ larger than the lead vehicle, i.e., $\Delta v_k = 5\;\rm m/s$. These initial conditions yield the initial distance headway which is unsafe (see Fig.~\ref{fig:highD_example}); this allows us to examine the ability of the algorithm to handle emergency conditions. Finally, in each test case, the ACC target speed $v_s$ is set to be the average speed of the sampled lead vehicle's speed trajectory $(v_k)_k$. Subsequently, the speed of the lead vehicle fluctuates near $v_s$, and we can test both the speed-tracking and car-following functionalities in one test case. One test case is shown in Fig.~\ref{fig:highD_example}.  

\begin{figure}[!h]
    \centering
    \vspace{0.5em}
    \includegraphics[width=0.47\textwidth]{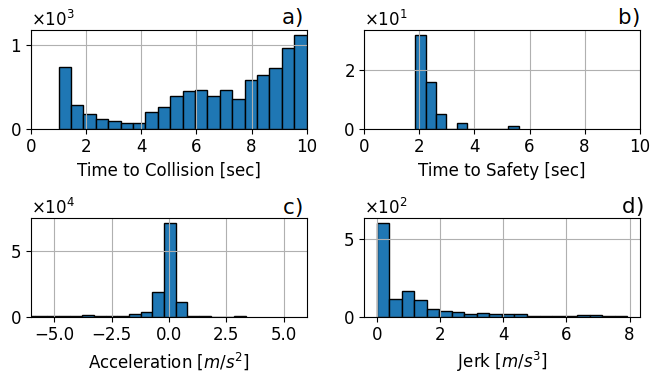}
    \caption{Statistics from the 56 test cases where each data point corresponds to a frame in the simulation where the frame rate is 100 Hz: (a) Time to Collision (ToC) is calculated as the time required for the follower and lead vehicles to collide assuming they travel at the current speed. (Infinite ToC values when the lead vehicle is faster than the follower are neglected) (b) Time to Safety is defined as the time elapsed from $t=0$ to the time instance $t_k$ when $d_k\geq d_s +T_sv_k$. (c) Acceleration commands from the ACC algorithm. (d) Jerk.}
    \label{fig:highD_results}
\end{figure}

As shown in Fig.~\ref{fig:highD_results}, our algorithm can ensure a sufficient Time-to-Collsion (ToC) that is larger than 4 seconds at most of the simulation time. We also note that the majority of the time when the ToC $\leq 2$ sec is due to the test cases being initialized with a small distance headway and large velocity difference. Meanwhile, our algorithm can also regulate the ego vehicle back to a safe distance headway within 4 seconds while the acceleration/deceleration effort is moderate and the jerk values are kept smaller than $2\;\rm m/s^3$ for a comfortable driving experience. 

\subsection{ACC in Out-Of-Distribution Scenarios}\label{subsec:result_ood}

\begin{figure}[!h]
    \centering
    \includegraphics[width=0.47\textwidth]{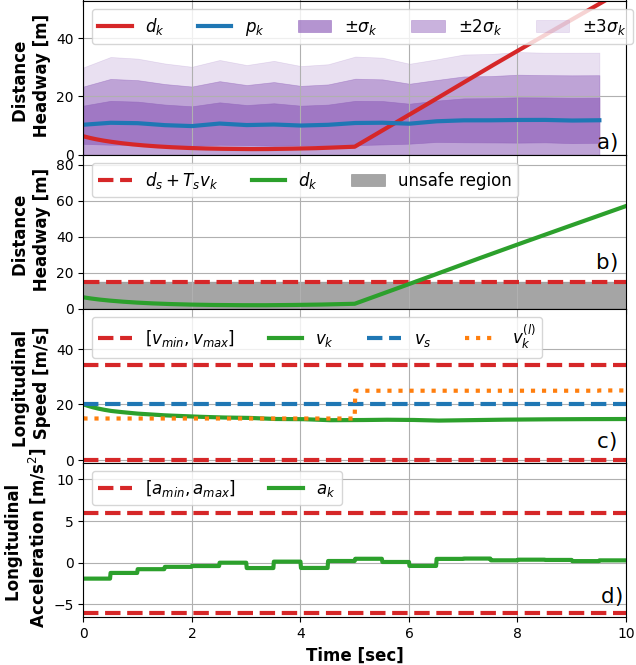}
    \caption{ACC example in an OOD scenario. The animation is available in \url{https://bit.ly/4ag8rih}.}
    \label{fig:acc_sim_ood}
\end{figure}

We note that the previous case studies are conducted using the same weather settings and the same lead vehicle model as in the training dataset (see Sec.~\ref{subsec:result_ensemble}). To further explore the capability of the algorithm in OOD scenarios, we change the lead vehicle model from a small 2020 Lincoln MKZ sedan (in black) to a large firetruck (in red). Moreover, we also change the weather from {\tt ClearNoon} to {\tt HardRainSunset} where the lighting condition is worse, objects cast shadows on the road and raindrops block the camera views. The animation is available in \url{https://bit.ly/4ag8rih}. To compare with the in-distribution scenario, we use the same initial distance headway, initial speed difference, and speed profile of the lead vehicle as in the example presented in Fig.~\ref{fig:acc_sim}. The results are reported in Fig.~\ref{fig:acc_sim_ood}. We note that the Deep Ensemble can capture the out-of-distribution and yield predictions with large variance. Then, Stochastic MPC commands the vehicle to take conservative maneuvers and decelerate to a speed lower than the set speed $v_s$.

\section{Conclusion}\label{sec:conclusion}
In this paper, we introduced a Deep Ensemble-based distance headway estimator using RGB images of the lead vehicle. This estimator provides both mean and variance of the headway distance. A Stochastic MPC based controller is then designed to enable adaptive cruise control with probabilistic safety. Using a high-fidelity simulator and real-world traffic dataset, we demonstrated the effectiveness of our proposed approach in speed tracking and car following, ensuring safety in both in-distribution and out-of-distribution scenarios.
\bibliographystyle{IEEEtran}
\bibliography{ref.bib}
\end{document}